\pdfoutput=1
\documentclass[sigconf]{acmart}

\usepackage[linesnumbered,ruled,vlined]{algorithm2e}
\usepackage{caption}
\usepackage{subcaption}

\AtBeginDocument{%
  \providecommand\BibTeX{{%
    \normalfont B\kern-0.5em{\scshape i\kern-0.25em b}\kern-0.8em\TeX}}}

\copyrightyear{2023}
\acmYear{2023}
\setcopyright{acmlicensed}
\acmConference[GECCO '23]{Genetic and Evolutionary Computation Conference}{July 15--19, 2023}{Lisbon, Portugal}
\acmBooktitle{Genetic and Evolutionary Computation Conference (GECCO '23), July 15--19, 2023, Lisbon, Portugal}
\acmPrice{15.00}
\acmDOI{10.1145/3583131.3590512}
\acmISBN{979-8-4007-0119-1/23/07}

\begin{document}

\title{Rethinking Population-assisted Off-policy Reinforcement Learning}

\author{Bowen Zheng}
\email{bowen.zheng@protonmail.com}
\affiliation{%
  \department{Department of Computer Science and Engineering}
  \institution{Southern University of Science and Technology}
  \streetaddress{1088 Xueyuan Blvd}
  \city{Shenzhen}
  \state{Guangdong}
  \country{China}
  \postcode{518055}
}

\author{Ran Cheng}
\authornote{Corresponding Author}
\email{ranchengcn@gmail.com}
\affiliation{%
  \department{Department of Computer Science and Engineering}
  \institution{Southern University of Science and Technology}
  \streetaddress{1088 Xueyuan Blvd}
  \city{Shenzhen}
  \state{Guangdong}
  \country{China}
  \postcode{518055}
}

\renewcommand{\shortauthors}{Zheng and Cheng}


\begin{abstract}
  While off-policy reinforcement learning (RL) algorithms are sample efficient due to gradient-based updates and data reuse in the replay buffer, they struggle with convergence to local optima due to limited exploration. On the other hand, population-based algorithms offer a natural exploration strategy, but their heuristic black-box operators are inefficient. Recent algorithms have integrated these two methods, connecting them through a shared replay buffer. However, the effect of using diverse data from population optimization iterations on off-policy RL algorithms has not been thoroughly investigated. In this paper, we first analyze the use of off-policy RL algorithms in combination with population-based algorithms, showing that the use of population data could introduce an overlooked error and harm performance. To test this, we propose a uniform and scalable training design and conduct experiments on our tailored framework in robot locomotion tasks from the OpenAI gym. Our results substantiate that using population data in off-policy RL can cause instability during training and even degrade performance. To remedy this issue, we further propose a double replay buffer design that provides more on-policy data and show its effectiveness through experiments. Our results offer practical insights for training these hybrid methods.
\end{abstract}




\maketitle

\section{Introduction}
Reinforcement learning (RL) algorithms have demonstrated remarkable success in a range of domains, including arcade games \cite{mnihPlayingAtariDeep2013,mnihHumanlevelControlDeep2015}, board games \cite{silverMasteringGameGo2016,silverMasteringChessShogi2017}, and robotic control \cite{lillicrapContinuousControlDeep2019,haarnojaSoftActorCriticOffPolicy2018,fujimotoAddressingFunctionApproximation2018}. The use of high-capacity artificial neural networks and gradient-based optimization methods has greatly contributed to this success. However, there are still challenges that hinder the use of these methods for more general tasks in real-life scenarios. For model-free RL, the first challenge is the instability of convergence and sensitivity to hyperparameters. The second challenge is the balance between exploration and exploitation, which is inherent in the bootstrapping learning pattern: the agent generates data from the environment, then updates its parameters based on that data, leading to the potential for unstable training and high sample complexity. The design of the reward function is also a critical factor, as it must effectively guide the learning process by providing meaningful signals.



Besides reinforcement learning, population-based methods such as evolutionary algorithms (EAs) are competitive alternatives for solving policy search tasks. Previous works included evolving the policy's topology and weights \cite{stanleyEfficientReinforcementLearning2002a}. Recent methods used evolution strategies (ESs) for direct policy search on weights and achieved similar or better results than gradient-based methods in arcade games and continuous control tasks \cite{salimansEvolutionStrategiesScalable2017,chrabaszczBackBasicsBenchmarking2018,liuTrustRegionEvolution2019,fuksEvolutionStrategyProgressive2019}. These methods use the agent's returns from the environment as the objective, providing a natural exploration strategy in the parameter space via the population. Moreover, the population's parallelization property also benefits training on the massively distributed framework \cite{baiLamarckianPlatformPushing2022,tangEvoJAXHardwareAcceleratedNeuroevolution2022,huangEvoXDistributedGPUaccelerated2023}. However, these methods often suffer from high sample complexity and struggle with slow convergence rates when optimizing neural networks due to their high dimensionality.

Consequently, combining these methods to form a hybrid approach is intriguing. One such pioneer work is \cite{khadkaEvolutionGuidedPolicyGradient2018}, which introduced the Evolutionary Reinforcement Learning (ERL) framework. The ERL framework employs an extra policy in addition to the genetically evolved population of policies. This additional policy, referred to as the target policy, is trained using off-policy RL methods with a shared replay buffer that stores recent transition data from all policies in the population and the target policy. The weights of the target policy guide the evolution, increasing the convergence speed, while the replay buffer facilitates information flow from the population to the RL method. This type of cooperative pattern is referred to as \emph{population-assisted off-policy RL}.

An essential aspect of ERL is that the population does not directly modify the parameters of the target policy; instead, it uses transitions from the population's trajectories to implicitly assist the off-policy RL method. However, the exact mechanism by which information is transferred from EA to RL, and its impact, is not yet fully understood and requires further exploration.
In our work, we have analyzed the formulation of off-policy RL algorithms commonly used in ERL frameworks and argue that the data from the population may introduce errors that harm the optimization process in the off-policy RL method due to a mismatch in distribution between the target policy and the population policies. These errors have been overlooked in previous studies.

We conduct an empirical investigation to assess how the population data influences the off-policy RL performance within the ERL framework. We propose a scalable training design that aligns training settings across algorithms, ensuring fair comparison. Then we develop a tailored ERL framework where its EA method generates trajectories with higher returns compared to the original ERL framework. And we design it in a minimalist way, reducing extraneous interference. By analyzing the action discrepancy between the target policy and population to performance on continuous robot locomotion tasks, we empirically confirm our hypothesis that errors from the off-policy population data can impair the target policy when the action discrepancy is considerable.



Furthermore, we propose a novel design to address the errors introduced by the population data. Our method utilizes two separate replay buffers to store transitions from the target actor and the population, respectively. The data is then mixed at a specified ratio to provide near-on-policy data for the off-policy RL optimization. Experiments on continuous robot locomotion tasks show that this modification makes the ERL framework more robust by reducing the impact of errors from the population data for the target actor.

Generally, our work evaluates the potential of off-policy RL in handling population data that deviates from the distribution of the target policy. We aim to shed light on this overlooked topic and provide insights into the design of reliable population-based algorithms.
The main contributions of this work are summarized as:
\begin{itemize}
    \item We analyze the off-policy RL method used in the population-assisted framework and observe a neglected flaw in the distribution mismatch between the population and the target actor.
    \item We propose a tailored ERL framework and test it under a uniform and scalable training design for different algorithms to empirically verify that this issue could deteriorate the target actor in the off-policy RL updates.
    \item By utilizing the correction effect of on-policy data, we propose a double replay buffer design for the ERL framework to remedy this issue.
\end{itemize}
\section{Background}
In this section, we present the notations and fundamental concepts of reinforcement learning, followed by an overview of the Evolutionary Reinforcement Learning (ERL) framework, a representative of population-assisted off-policy RL algorithms.

\subsection{Notation}
\label{sec:notation}
In RL, the problem is framed under the Markov Decision Process (MDP) assumption, which is defined as a tuple $(\mathcal{S},\mathcal{A},p,r)$. Here, $\mathcal{S}$ is the state space, $\mathcal{A}$ is the action space, $p: \mathcal{S} \times \mathcal{S} \times \mathcal{A} \mapsto [0,1]$ represents the transition probability between states and actions, and $r: \mathcal{S} \times \mathcal{A} \mapsto [r_{\text{min}},r_{\text{max}}]$ defines the reward received after each transition. The agent's policy, $\pi(a_t|s_t)$, maps states to actions. The goal of RL is to optimize the policy $\pi$ to maximize the discounted return along the trajectory $\tau$:
\begin{equation}
    \label{eq:rl_objective}
    \mathbb{E}_{\tau \sim p(\tau|\pi)} \left[ \sum_{t=0}^\infty \gamma^t r_t \right], 
\end{equation}
where $\gamma \in [0,1]$ is the discount factor, and $p(\tau|\pi)$ is the distribution of the trajectory generated by policy $\pi$. Note that the discount factor is used as a technical trick in episodic tasks and the ultimate goal is still to maximize the un-discounted return.

In contrast to supervised deep learning algorithms, model-free Deep Reinforcement Learning (DRL) has a unique learning pattern characterized by the use of current parameters to generate data for the next update. The \emph{behavior policy} is the policy used to collect data from the environment, while the \emph{target policy} is the policy that leverages this data for learning. If the target policy is the same as the behavior policy, the learning is referred to as \emph{on-policy} learning. On the other hand, \emph{off-policy} learning occurs when the behavior and target policies are different.

In the Actor-Critic setting, a popular learning pattern in DRL, the combination of policy-based methods \cite{williamsSimpleStatisticalGradientfollowing1992} and value-based methods \cite{watkinsQlearning1992,mnihPlayingAtariDeep2013} results in faster convergence. The \emph{actor} refers to the policy $\pi$, while the \emph{critic} refers to the estimation of either the state-value function or the action-value function under the actor. The action-value is defined as
\begin{equation}
    Q^\pi(s_t,a_t)=\mathbb{E}_{\tau \sim p(\tau|\pi)} \left[ \sum_{k=0}^\infty \gamma^k r_{t+k} |s_t,a_t\right].
\end{equation}

\subsection{Evolutionary Reinforcement Learning Framework}
\label{sec:erl_framework}
Traditional RL methods often struggle with diverse exploration and brittle convergence. As a result, population-based methods like Evolutionary Algorithms (EAs) have emerged as viable alternatives. These methods conduct the direct policy search based on the agent's returns from the environment and inherently include exploration strategies within the parameter space. Unlike gradient-based methods, EAs do not require explicit consideration of the internal temporal structure of episodes, allowing them to discard additional components of traditional RL methods, such as the Markov decision process (MDP) assumption and the discount factor. However, population-based methods are typically less efficient than gradient-based optimizers and have high sample complexity due to their inability to utilize internal information from episodes.


The Evolutionary Reinforcement Learning (ERL) framework incorporates the complementary benefits of Evolutionary Algorithms (EAs) and off-policy Reinforcement Learning (RL) methods for improved performance \cite{khadkaEvolutionGuidedPolicyGradient2018}. The framework integrates EA into off-policy RL by performing EA and off-policy RL optimization simultaneously and using a replay buffer to store informative experiences generated from the evolutionary stage and the RL stage. The replay buffer helps optimize the off-policy RL, while the weights of the RL agent boost the convergence speed of EA.

ERL maintains a population of actors along with an additional target actor and its corresponding critic. All actors are randomly initialized, and at each iteration, they are evaluated by performing episodes, where their fitness is calculated as the corresponding return. Experiences are recorded in a fixed-capacity replay buffer. A genetic algorithm is applied to the population based on their fitness, while the target actor collects episodes of experiences in the replay buffer. The target actor and critic are then updated using off-policy gradient-based methods, and the worst actor in the population is replaced with the target actor after several iterations to provide feedback from RL to the evolutionary population. A variant of ERL \cite{pourchotCEMRLCombiningEvolutionary2019} uses multiple independent target agents trained by off-policy RL and participating in the evolution. However, the framework still follows the general idea described above.

Although some methods \cite{colasGEPPGDecouplingExploration2018,nilssonPolicyGradientAssisted2021,baiEvolutionaryReinforcementLearning2023} also utilize a similar shared replay buffer strategy with population data, they introduce task-dependent goal space or behavioral characterizations and propose new objectives. In this paper, to streamline the discussion, we concentrate on methods that solely focus on maximizing the expected episodic return, like ERL.

\section{Effect of Off-policy Population Data}
\label{sec:effect_of_pop_data}

In this section, we investigate the underlying mechanism of Evolutionary Reinforcement Learning (ERL) frameworks and assess the impact of population experiences on their performance. We establish a uniform and scalable training design and develop a tailored ERL framework to eliminate extraneous factors and precisely measure the effect of population experiences under varying degrees of off-policy learning. Our approach's efficacy is demonstrated through experiments on robot locomotion tasks.

\subsection{Observation and Motivation}
\label{sec:obs}

The ERL framework combines the strengths of both Evolutionary Algorithms (EAs) and Reinforcement Learning (RL) through direct weight injection from RL to EA and indirect trajectory information from EA to RL. The direct weight injection enables the EA to leverage the weights of the high-performing target actor from RL as guidance, thus accelerating convergence and enhancing the overall population performance. This mechanism is relatively straightforward and contributes to the ERL framework's performance.

In contrast, the mechanism from EA to RL is less transparent, as it uses the collected experiences from EA to indirectly assist the RL updates instead of employing the weights from EA. This implicit information transfer benefits by reusing the learning pattern of off-policy RL algorithms without additional modification, as only the input data from the replay buffer differs. 
And the original work \cite{khadkaEvolutionGuidedPolicyGradient2018} claims that two factors let the experiences from EA assist the RL optimization. First, the diverse exploration through experiences from the population plays a crucial role in assisting RL, where the population of actors explores the parameter space, complementing the exploration of the target actor in action space. Second, the experiences from evolutionary iterations form an implicit prioritization for higher long-term payoff since EA directly optimize agents by their episodic return.
ERL and its variants \cite{khadkaEvolutionGuidedPolicyGradient2018,pourchotCEMRLCombiningEvolutionary2019,khadkaCollaborativeEvolutionaryReinforcement2019,bodnarProximalDistilledEvolutionary2020,tangGuidingEvolutionaryStrategies2021} follow this insight and train the target agent(s) directly through the replay buffer with data from the population and itself. However, they neglect that the success of the assistance from EA to RL is based on a subtle implicit assumption that off-policy RL algorithms are capable of taking advantage of these experiences. Although this assumption is crucial for the ERL framework, whether it holds is scarcely discussed.

We begin by conducting a thorough examination of prior research on the ERL framework and its variants. The majority of these studies \cite{khadkaEvolutionGuidedPolicyGradient2018,khadkaCollaborativeEvolutionaryReinforcement2019,bodnarProximalDistilledEvolutionary2020,leeEfficientAsynchronousMethod2021} employ off-policy deterministic actor-critic methods, particularly Deep Deterministic Policy Gradient (DDPG) \cite{lillicrapContinuousControlDeep2019} or Twin Delayed Deep Deterministic policy gradient (TD3) \cite{fujimotoAddressingFunctionApproximation2018}.
On continuous action spaces, the use of deterministic policies is preferred for their increased sample efficiency. Deterministic policies output a single, deterministic action given a state, rather than a distribution of actions. The Deterministic Policy Gradient (DPG) method \cite{silverDeterministicPolicyGradient2014} outlines the gradient calculation for a deterministic policy $\mu_\theta(s)$:
\begin{equation}
    \nabla_\theta J(\mu_\theta(s)) \doteq \mathbb{E}_{s \sim d_\mu} \left[ \nabla_\theta \mu_\theta(s) \nabla_a Q^\mu(s,a)|_{a=\mu_\theta(s)}  \right],
    \label{eq:onpolicy_dpg}
\end{equation}
\sloppy{where $Q^\mu(s,a)$ is the true action-value function and $d_\mu$ is the normalized form of the discounted state visitation frequency $\sum_{t=0}^{\infty} \gamma^t P\left(\\s_t=s|\mu\right)$. The off-policy version of DPG provides an approximation to the above gradient calculation: }
\begin{equation}
    \nabla_\theta J(\mu_\theta) \approx \mathbb{E}_{s \sim d_b} \left[ \nabla_\theta \mu_\theta(s) \nabla_a Q^\mu(s,a)|_{a=\mu_\theta(s)}  \right],
    \label{eq:offpolicy_dpg}
\end{equation}
where $d_b(s)$ is the marginal state distribution of the behavior policy $b(s)$. This approximation becomes accurate when the state distribution between the target policy and behavior policy is approximately stationary, i.e. $d_\mu(s) \approx d_b(s)$.
The true action-value function $Q^\mu$ can be estimated using a parametric model $Q_\phi$, updated by the Bellman operator $\mathcal{B}^\mu$:
\begin{equation}
    \begin{aligned}
        J(Q_\phi) =  \mathbb{E}_{(s,a) \sim d} \left[ Q_\phi(s,a) - \mathcal{B}^\mu Q_{\phi} \right]^2 \\
        \mathcal{B}^\mu Q_{\phi}=r+\gamma \mathbb{E}_{p(s'|s,a)} \left[ Q_{\phi}(s',\mu(s')) \right],
    \end{aligned}
    \label{eq:offpolicy_q_update}
\end{equation} 
where $d$ is some state-action distribution.


As formulated above, the off-policy methods utilize a replay buffer $\mathcal{D}$ that stores experiences $(s_t, a_t, r_t, s_{t+1})$ from previous iterations. Batches of data are then sampled from the replay buffer to update the actor $\mu_\theta$ and critic $Q_\phi$. The distribution $d_b(s)$ and $d(s,a)$ becomes the sampling distribution from $\mathcal{D}$, which we denote as $d_\mathcal{D}$. For simplicity, we use $d_\mathcal{D}(s)$ to represent the marginal state distribution, which is calculated as $\mathbb{E}_{a \sim \mathcal{D}}[d_\mathcal{D}(s,a)]$.

However, these off-policy methods were initially designed for situations where the policy explores with small perturbations on actions, and the replay buffer stores recent experiences from previously updated policies. Consequently, the data sampled from the replay buffer is often highly correlated with the current policy, and the distribution $d_{\mathcal{D}}$ is nearly equal to the corresponding on-policy distribution. In the case of ERL frameworks, however, the off-policy RL method must cope with external experiences generated by actors in the population, which can significantly diverge from the target actor during the evolutionary process.

\begin{proposition}
	\label{prop:mixed_policy_gradient}
	Mixing off-policy data into the policy gradient with the ratio $\alpha$ will changes the deterministic policy gradient from $\mathbb{E}_{s \sim d_\mu} \left[ \nabla_\theta Q^\mu(s, \mu_{\theta}(s))  \right]$ to $\mathbb{E}_{s \sim d_\mu} \left[ \nabla_\theta Q^\alpha(s,\mu_{\theta}(s)) \right]$, where $\rho(s)=\frac{d_b(s)}{d_\mu(s)}$ and
    \begin{equation}
        \label{eq:mixed_Q}
        Q^\alpha(s,a)=Q^\mu(s,a)+\alpha(\rho(s)-1)Q^\mu(s,a).
    \end{equation}
\end{proposition}

\begin{proof}
	See Appendix \ref{appendix:props}.
\end{proof}

In ERL frameworks, the actor updated by the off-policy methods can be biased towards unknown directions in non-tabular cases if the mixed data in the replay buffer leads to a mismatch between $d_\mu(s)$ and $d_\mathcal{D}(s)$. Proposition~\ref{prop:mixed_policy_gradient} indicates that mixing off-policy data in a ratio $\alpha$ is equivalent to adding a regularization term $(\rho(s)-1)Q^\mu(s,a)$ to the action-value estimate with weight $\alpha$ in the policy gradient, which can result in biased training. This regularization term becomes zero only when the population data distribution is close to $d_\mu(s)$. According to \cite{zimmerExploitingSignAdvantage2019}, under certain assumptions, the difference in the marginal state distribution of the behavior policy and target policy $|d_\mu(s) - d_b(s)|$ for a given state $s$ is bounded by $\max_{s \in \mathcal{S}} \| \mu(s) - b(s) \|_2$. This suggests that when the difference in actions between the target actor and individuals in the population is small, the mismatch in the marginal state distribution is reduced, leading to lesser error from the population data.


For the critic updates, although the off-policy formula in \eqref{eq:offpolicy_q_update} permits using arbitrary state-action pair distribution $d$, the critic can still be inaccurate in practice, exacerbating subsequent updates of the target actor. With function approximation, the mean square loss under $d_\mu(s,a)$ or $d_\mathcal{D}(s,a)$ has different optimization preferences for the critic. Additionally, the Bellman operator is approximated based on the data from the replay buffer, where the infinite state-action visitation assumption may not hold, and thus, the distribution mismatch can lead to incorrect updates. Theoretical works in \cite{suttonPolicyGradientMethods2000,silverDeterministicPolicyGradient2014,nachumSmoothedActionValue2018,sinhaExperienceReplayLikelihoodfree2022} recommend using a near-on-policy distribution for optimal performance for $Q$ functions, particularly in tasks with continuous action space.


As a result, we believe that employing these off-policy RL algorithms with experiences from the population could potentially impair performance due to the experiences distribution mismatch between the population policies and the target policy. However, the diverse experiences of the population also provide exploration benefits by expanding the agent's perception of the environment and prioritizing areas with high long-term returns through evolutionary pressure. When using the ERL replay buffer with mixed data, there is an implicit and delicate balance between these benefits and errors. This trade-off is challenging to formulate in equations and involves many uncertain factors, so we investigate it by conducting thorough experiments in the following sections.




\subsection{Uniform Scalable Training Design for Fair Comparison}
\label{sec:training_design}

In order to understand the impact of population experiences on off-policy reinforcement learning, it is essential to conduct fair comparisons between different algorithms using a consistent training pattern. Previous studies (e.g., \cite{khadkaEvolutionGuidedPolicyGradient2018,khadkaCollaborativeEvolutionaryReinforcement2019,bodnarProximalDistilledEvolutionary2020,leeEfficientAsynchronousMethod2021}) have compared RL and ERL algorithms with a varying balance between exploration and exploitation. In each iteration, the RL algorithm interacts with the environment for one step, adds the experience to the replay buffer, and then updates the policy by sampling a batch of experiences from the replay buffer. On the other hand, the ERL algorithm collects episodes from the population and the target actor in each iteration. It counts the total number of timesteps collected in that iteration and performs the same number of gradient updates. Even though the ERL algorithm employs the same gradient-based method as RL, its sampling and update intervals are significantly different. For example, if the ERL algorithm has a population size of 10, and each actor collects an episode of 1000 timesteps, the total number of gradient updates in each iteration would be $(10 + 1) \times 1000 = 11000$. This training design makes it challenging to determine whether the performance difference is due to the population or the different sampling and update intervals.

To fairly compare the impact of population experiences on off-policy RL, we propose a uniform training design. Our framework, shown in Fig.~\ref{fig:distributed_ERL}, allows for scalable implementation and ensures the equal comparison of different algorithms. We use remote rollout workers to collect trajectories from actors in parallel and store them in local replay buffers for RL updates. For ERL algorithms, we use a single rollout worker for the target actor and perform the same number of gradient updates as the number of timesteps collected from the target actor. For RL algorithms, we deploy an equal number of rollout workers (equal to the population size + 1) for the target actor and perform the same number of gradient updates as the number of timesteps collected from the first rollout worker. These workers are synchronized with the latest parameters of the target actor. Thus, we align different types of algorithms, and compared to RL methods, ERL methods only include an additional evolution procedure in every iteration. The experiment terminates when a certain number of RL training steps (i.e., number of gradient updates) is reached.

\begin{figure}[t]
	\centering
	\includegraphics[width=0.8\linewidth]{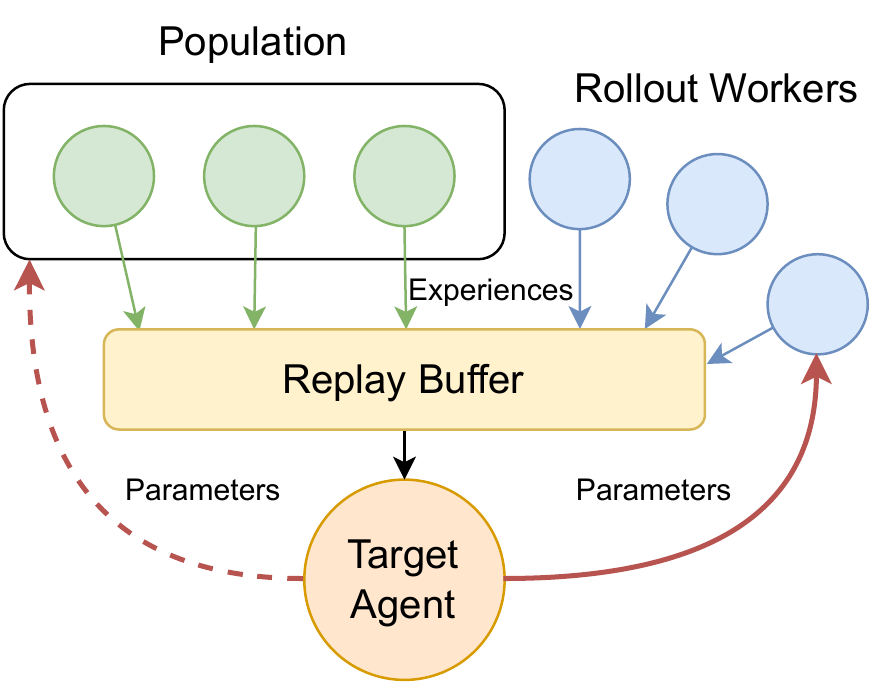}
	\caption{The uniform scalable training design for ERL and RL algorithms. Green circles represent rollout workers for the population, where each worker involves one actor in the population. Blue circles represent rollout workers for clones of the target actor.}
    \label{fig:distributed_ERL}
\end{figure}

In this way, we control different types of algorithms running at a similar exploration and exploitation trade-off level, and metrics over training steps can be reasonably aligned. In addition, we decouple the training steps from the EA part. Now, the number of training steps at each iteration is not directly relevant to the population size and the episode length of each individual anymore. Lastly, compared to using sampled timesteps as termination conditions in previous designs, which resulted in only hundreds of evolution iterations during training, our design also permits significantly more evolution iterations, which will unleash the potential of the population. Although we acknowledge this would collect more timesteps, the scalable training design dramatically amortizes the sampling time, fully utilizing ERL's parallel nature, and is an acceptable cost for our later simulation tasks in practice. We note that the choice of not using the total sampled timesteps as the horizontal coordinate has also been applied to other hybrid methods \cite{nilssonPolicyGradientAssisted2021,marchesiniGeneticSoftUpdates2021}.

\subsection{Tailored Evolutionary Reinforcement Learning Framework}
\label{sec:tailored_ERL}

With the aforementioned training design, we evaluate the performance of the original ERL. The EA part of the original ERL is a Genetic Algorithm (GA) with elitism, using Gaussian noise-based mutation and n-point crossover as variation operators on policy weights. We follow the official implementation\footnote{\url{https://github.com/ShawK91/Evolutionary-Reinforcement-Learning/tree/neurips_paper_2018}} of GA and utilize the same hyperparameter values. To address the critic overestimation issue in the RL part, we substitute the off-policy RL algorithm DDPG with the more recent TD3 method.

For clarity, we denote the ERL algorithm with the original GA as \emph{ERL-GA} and the training with the TD3 algorithm as \emph{no-pop}. The results of these two training methods are shown in Fig.~\ref{fig:erl_ga_reward}, indicating that \emph{no-pop} outperforms \emph{ERL-GA} on the target actor. Furthermore, we plot the average fitness across different trials and select one trial for each task to illustrate the fitness distribution around 1M, 2M and 3M training steps. As demonstrated in Fig.~\ref{fig:erl_ga_fitness}, many experiences gathered from the population originate from trajectories with low returns. The double-peak fitness distribution implies that, through the GA, numerous actors in the population become less adept compared to the target actor. Consequently, utilizing these low-return experiences may decelerate the convergence speed of the target actor and lead it towards a suboptimal solution. A further discussion can be found in Appendix~\ref{appendix:ga}.

\begin{figure}[h]
    \centering 
    \begin{subfigure}[t]{0.49\linewidth}
       \centering
       \includegraphics[width=\linewidth]{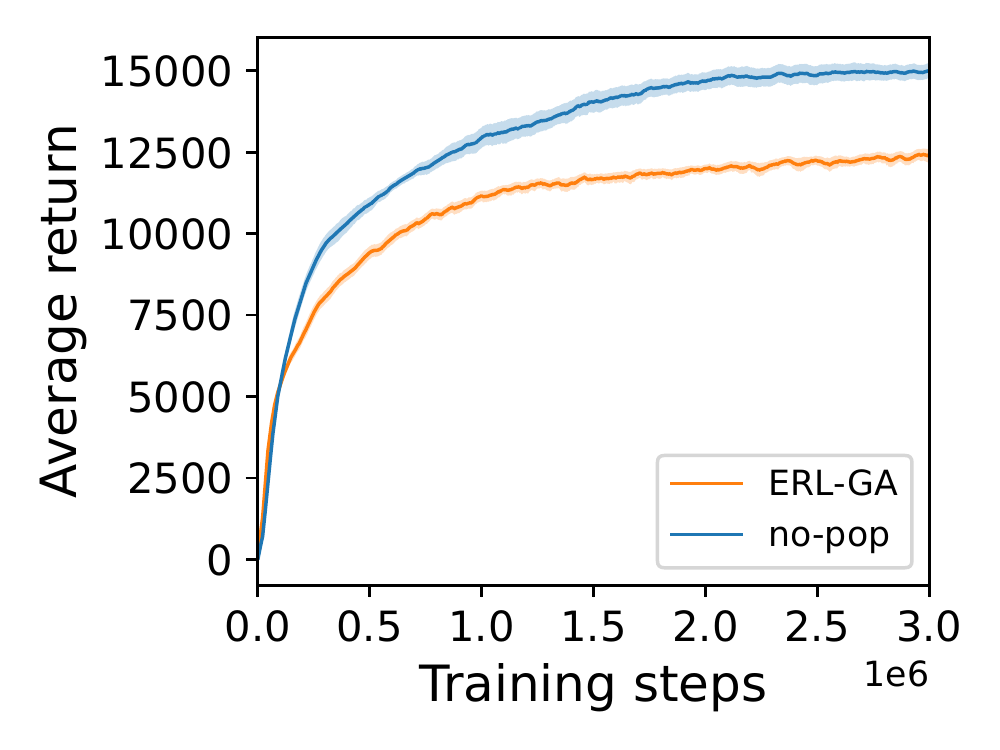} 
       \caption{HalfCheetah-v3}
    \end{subfigure}
    \begin{subfigure}[t]{0.49\linewidth}
       \centering 
       \includegraphics[width=\linewidth]{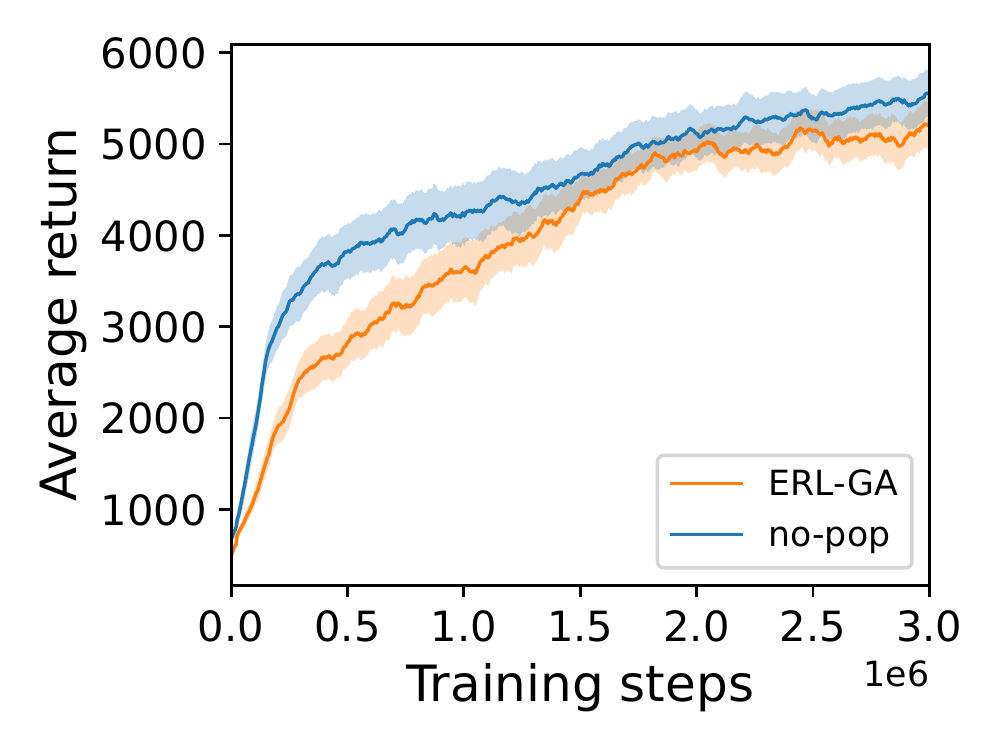} 
       \caption{Ant-v3}
    \end{subfigure}
    \caption{Learning curves for the original ERL (\emph{ERL-GA}) and parallel TD3 (\emph{no-pop}). The average returns of the target actor are aligned with its training steps.}
    \label{fig:erl_ga_reward}
\end{figure}


Hence, to examine the impact of population data, we seek an ideal Evolutionary Algorithm (EA) capable of generating high-quality trajectories from all actors in the population, eliminating the influence of poor population data, and allowing us to focus on the distribution mismatch error discussed in Section~\ref{sec:obs}. Moreover, in contrast to the GA method used in \emph{ERL-GA}, which involves numerous hyperparameters, we aim for a minimalist EA to reduce irrelevant complexity.

\begin{figure}[!t]
    \centering
    \begin{subfigure}[t]{\linewidth}
       \centering
       \includegraphics[width=\linewidth]{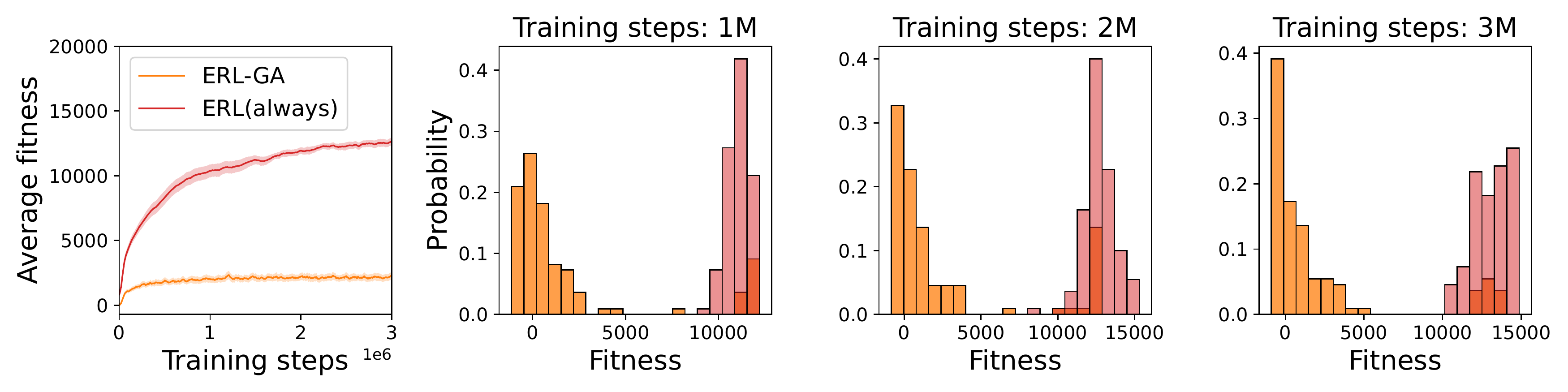} 
       \caption{HalfCheetah-v3}
    \end{subfigure}
    \begin{subfigure}[t]{\linewidth}
       \centering 
       \includegraphics[width=\linewidth]{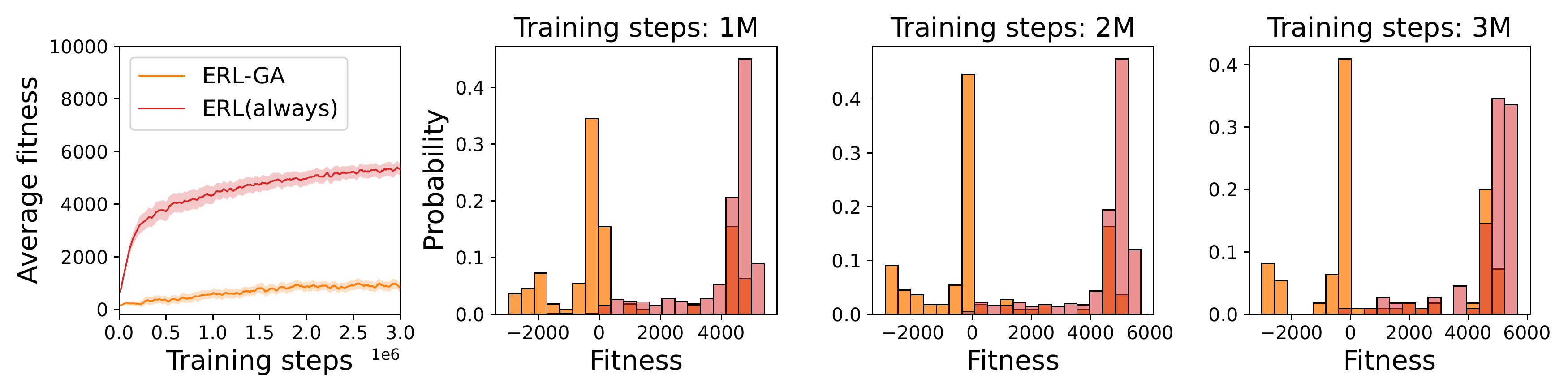} 
       \caption{Ant-v3}
    \end{subfigure}
    \caption{Training results on the original ERL (\emph{ERL-GA}) and our tailored ERL with \emph{always} strategy (\emph{ERL (always)}). The leftmost plots show the average fitness during the training. The other plots on the right represent the fitness distribution around 1M, 2M, and 3M training steps. }
    \label{fig:erl_ga_fitness}
\end{figure}

As a result, we introduce a tailored ERL framework that addresses the limitations of the previous GA method. This framework streamlines the training process and solves the issues found in \emph{ERL-GA}. It features an efficient EA that all population actors remain comparable to the target actor during evolution, thus high-return trajectories of the population are collected and stored in the replay buffer, as illustrated in Fig.~\ref{fig:erl_ga_fitness}. In addition, through the new EA, we can control the off-policy degree of the population, which is helpful to determine how different extents of distribution mismatch influence the target actor's updates.


\begin{algorithm}[b]
    \SetKwFunction{Fn}{Evaluate}
    \KwIn{initial policy parameters $\theta$, Q-function parameters $\phi$, replay buffers $\mathcal{D}$, ES hyperparameters: $(N, K, \sigma)$ }
    \KwOut{$\theta, \theta_{\text{pop}}$}

    Initialize population mean: $\theta_{\text{pop}}=\theta$\;
    
    \For{$iteration = 1,2,...$}{
        \For{$i \gets 1,...,N$}{
            Sample noise: $\epsilon_i \sim \mathcal{N}(0,I)$\;
            $\theta_i = \theta_{\text{pop}} + \sigma * \epsilon_i$\;
            $f_i \gets$ \Fn{$b_{\theta_i}$, $\mathcal{D}$}\;
        }
        $f_{\text{target}} \gets$ \Fn{$\pi_\theta$, $\mathcal{D}$}\;
        $\epsilon_{\text{target}} = (\theta - \theta_{\text{pop}}) / \sigma$\;

        \tcp{EA update}
        Update $\theta_{\text{pop}}$ by $\sigma, (\epsilon_i,f_i)_{i=1...N}$ and $(\epsilon_{\text{target}},f_{\text{target}})$\;

        \tcp{off-policy RL update}
        Update $\theta,\phi$ by off-policy RL with $\mathcal{D}$\;
    }
    \caption{The Tailored Evolutionary Reinforcement Learning}
    \label{algo:ERL}
\end{algorithm}

\begin{figure*}[t]
    \centering 
    \begin{subfigure}{0.24\linewidth}
       \centering
       \includegraphics[width=\linewidth]{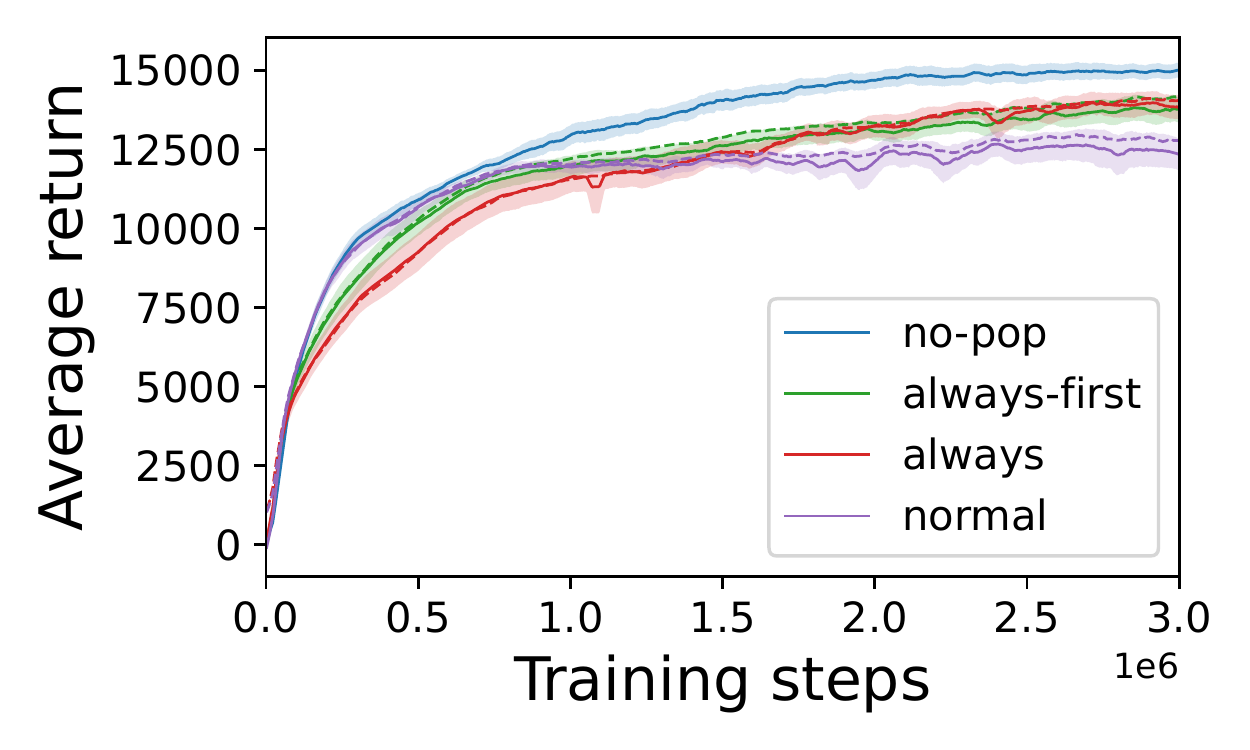} 
       \caption{HalfCheetah-v3}
    \end{subfigure}
    \begin{subfigure}{0.24\linewidth}
       \centering 
       \includegraphics[width=\linewidth]{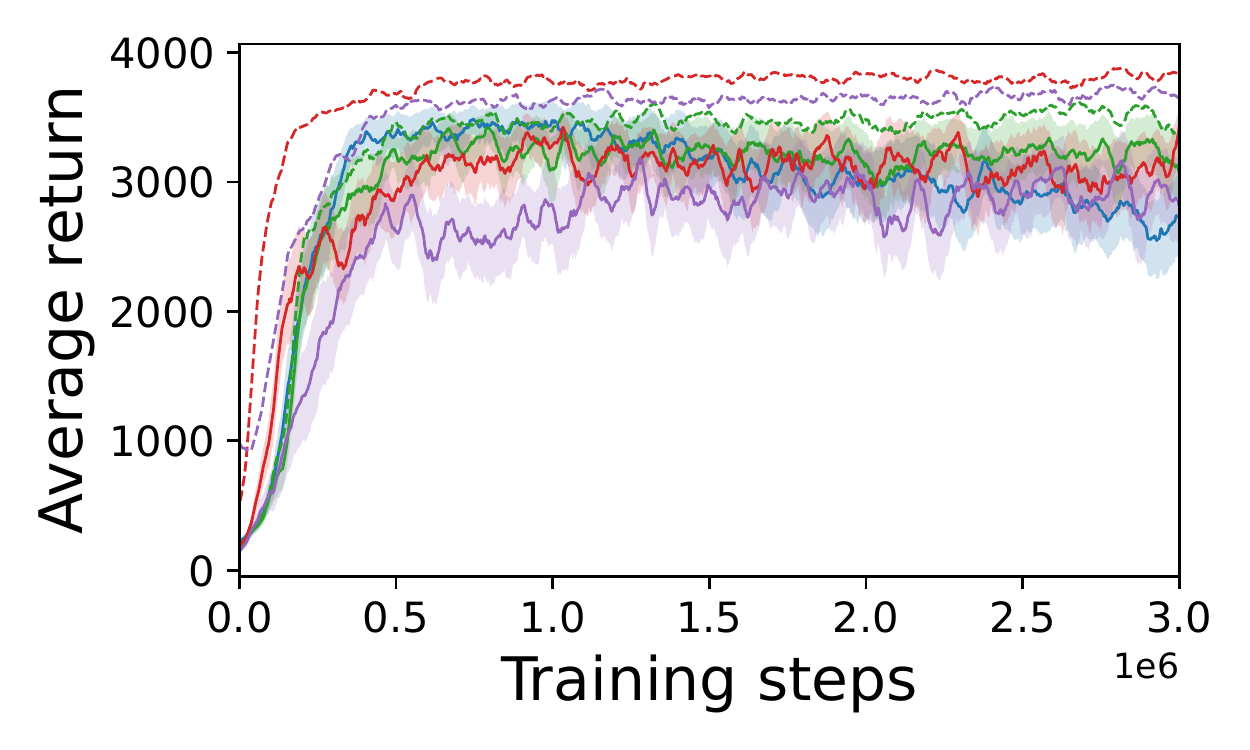} 
       \caption{Hopper-v3}
    \end{subfigure}
    \begin{subfigure}{0.24\linewidth}
        \centering 
        \includegraphics[width=\linewidth]{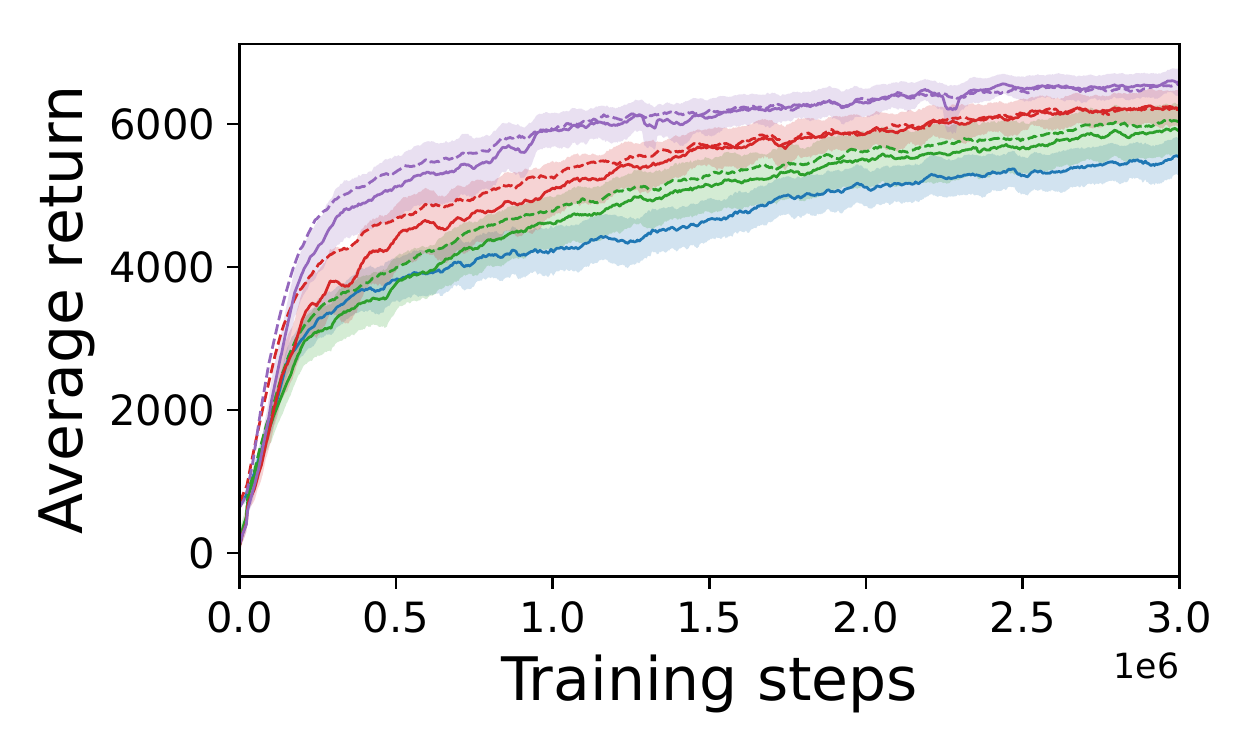} 
        \caption{Ant-v3}
     \end{subfigure}
     \begin{subfigure}{0.24\linewidth}
        \centering 
        \includegraphics[width=\linewidth]{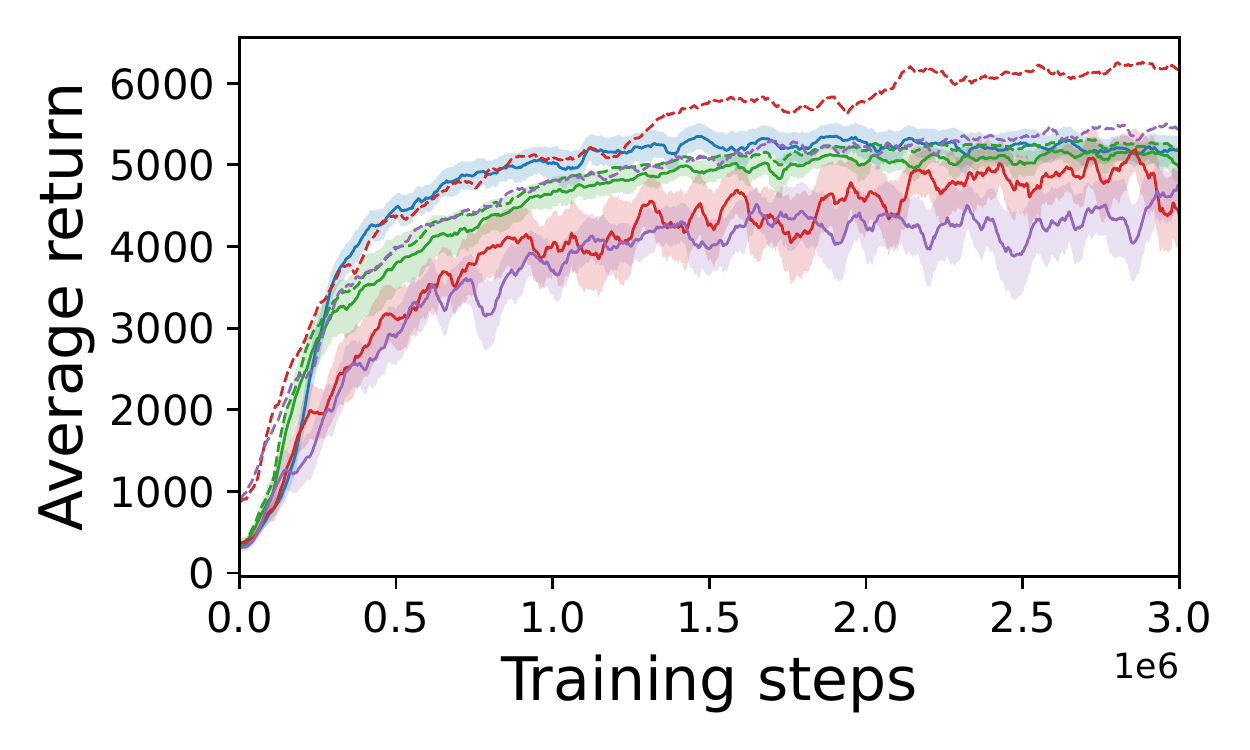} 
        \caption{Walker2d-v3}
     \end{subfigure}
    \caption{Learning curves on Mujoco locomotion tasks. For ERL algorithms (always-first, always, normal), the average returns of the population mean are drawn by dash lines with the same corresponding color.}
    \label{fig:erl_results}
\end{figure*}

We replace the original GA method with an ES method based on \cite{chrabaszczBackBasicsBenchmarking2018}. ES methods have several advantages in the field of RL: first, each individual is generated by simply adding Gaussian noise, which has been proven effective for neural networks \cite{salimansEvolutionStrategiesScalable2017,chrabaszczBackBasicsBenchmarking2018}; second, their population update methods are more stable for noisy fitness evaluation. Furthermore, our selected evolution method has minimal hyperparameters and no extra tricks, reducing unnecessary interference in our investigation.

We present the pseudo-code for our tailored ERL in Algorithm~\ref{algo:ERL}. Each iteration begins by generating $N$ individuals from an isotropic Gaussian distribution with mean $\theta_{\text{pop}}$ and fixed standard deviation $\sigma$. We use $\epsilon_{\text{target}} = (\theta - \theta_{\text{pop}}) / \sigma$ as a fake noise for the target actor.
In each iteration, each individual in the population is evaluated in a single episode, where the corresponding return is used as its fitness. The target actor is also evaluated in an episode to determine its fitness $f_{\text{target}}$. Then all experiences are stored in the replay buffer.

And we design three strategies to integrate our ES into the ERL framework, replacing the original process from RL to EA. After fitnesses are calculated, we update the population according to the following equation:
\begin{equation}
    \theta_{\text{pop}} \gets \theta_{\text{pop}}+\sigma \cdot \sum_{j=1}^{\hat{K}} (w_j \cdot \epsilon_j),
    \label{eq:es_update}
\end{equation}
where $\hat{K}$ is the number of parents, and $w_i = \frac{\log(\hat{K}+0.5)-\log(i)}{\sum_{j=1}^{\hat{K}}\log(\hat{K}+0.5)-\log(i)}$ are the recombination weights \cite{rudolph1997convergence} determined by the order of parents. Then the selection of parents and their order is determined by one of the strategies:
\begin{itemize}
    \item \textbf{normal}: directly sort individuals in the population and the target actor according to their fitnesses, and select the top $K+1$ actors as parents.
    \item \textbf{always}: sort individuals in the population by fitnesses and select the top $K$ actors as parents; then add the target actor to the parent set with the order according to its fitness.
    \item \textbf{always-first}: sort individuals in the population by fitnesses and select the top $K$ actors as parents; then add the target actor to the parent set as the first place, regardless of its real fitness.
\end{itemize} 
The three strategies control the preference of the population towards the target actor, thus indirectly determining the off-policy degree of the population. The order of constraint degrees from weak to strong is \emph{normal}, \emph{always}, and \emph{always-first}. And in comparison to \emph{ERL-GA}, the frequency of the transition from RL to EA is increased to every iteration, allowing the ES optimization to fully utilize the weights of the target actor.


\subsection{Empirical Analyses}
\label{sec:empirical_analyses}

For the test environments, although the ERL framework was initially designed for tasks with sparse reward signals, it was frequently tested on non-sparse reward tasks, such as the Mujoco locomotion tasks \cite{todorovMuJoCoPhysicsEngine2012} -- a standard benchmark for continuous environments in OpenAI Gym. Following these works, our experiments also focus on these Mujoco tasks.

We compare the tailored ERL framework with the TD3 algorithm \emph{no-pop}. To quantify the degree of distribution mismatch, we calculate the action discrepancies between the target actor and population actors. 
Since $\max_{s \in \mathcal{S}} \| \mu(s) - b(s) \|_2$ is practically untraceable, we compute the mean square error of the output actions under the states of the trajectory $\tau_{b_i}$ from an individual $b_i$ as the average action discrepancy between $\mu$ and $b_i$: 
\begin{equation}
    \delta(\mu, b_i) = \frac{1}{|\tau_{b_i}|} \sum_t \left[\mu(s_t)-b_i(s_t)\right]^2,
\end{equation}
representing the proximity of experience distribution between the target actor and the population. We also record the performance of an actor based on the weights of the population distribution mean as a measurement of the EA part in the ERL framework. Other implementation details remain the same as those in Section~\ref{sec:experiments}.

\begin{figure}[b]
    \centering
    \includegraphics[width=0.98\linewidth]{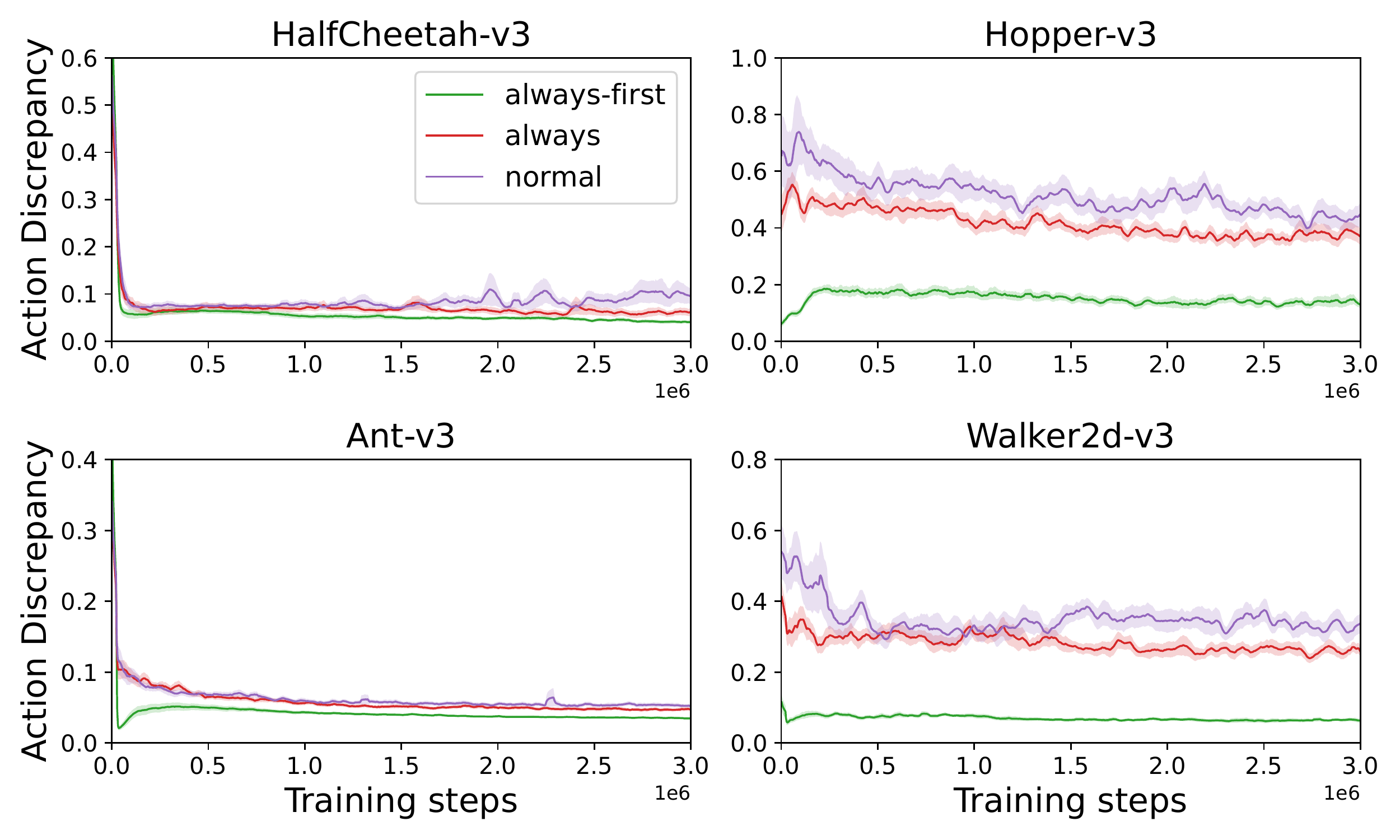}
	\caption{The average action discrepancy between the target actor and population actors in the tailored ERL framework.}
    \label{fig:ERL_similarity}
\end{figure}

The average return is displayed in Fig.~\ref{fig:erl_results}, and Fig.~\ref{fig:ERL_similarity} presents the average action discrepancy values across the population during training. By combining these results, we observe that the performance of the population mean consistently surpasses that of its target actor, indicating a successful information transfer from RL to EA.
In HalfCheetah-v3 and Ant-v3, the target actor achieves comparable performance to the population mean, where the action discrepancy is small across the three strategies. In contrast, although the performance of the population mean is high in Hopper-v3 and Walker2d-v3, implying the generation of high-return trajectories, the target actor struggles to effectively learn from these population data to achieve comparable performance when the action discrepancy is significant.

Moreover, the results of \emph{always} and \emph{normal} from Walker2d-v3 demonstrate that, although the population generates experiences with higher returns, the target actor of ERL performs worse; as the action discrepancy increases, the target actor undertakes more performance deterioration and becomes increasingly unstable during training. This evidence provides a counterexample and validates our argument that, under the trade-off between the benefits and errors of population data, these deviated data could hinder the target actor's learning via off-policy RL.

\section{Vanilla Remedy Method}

In Section~\ref{sec:effect_of_pop_data}, we discovered that using population experiences to assist the target actor's learning via the off-policy RL algorithm can be problematic within the ERL framework. Despite high returns, the presence of deviated off-policy population experiences in the replay buffer introduces errors in off-policy updates that may outweigh their benefits, leading to unstable training and reduced performance.


In the ERL framework, the replay buffer stores experiences from all actors, including near-on-policy experiences from previous target actors and off-policy experiences from population actors. Off-policy RL algorithms use uniformly sampled data from the replay buffer to update the target actor and its critic. However, as the population size increases, the proportion of near-on-policy data from the target actor in the replay buffer decreases. To alleviate the error caused by the distribution mismatch, it is essential to incorporate more on-policy data from the target actor into the updates.

To ensure efficient learning in the ERL framework, we implement two distinct replay buffers for experiences from the population and the target actor. During the sampling process, a specified percentage of data from the target actor's replay buffer is remixed. For each off-policy update, we sample a proportion $m \in (0,1)$ of experiences from the target replay buffer $\mathcal{D}_{\mu}$ and the remaining $(1-m)$ proportion from the population replay buffer $\mathcal{D}_{\text{pop}}$. These experiences are combined to form the final batch, which serves as input for the update. Consequently, the distribution of the final batch becomes
\begin{equation}
	\label{eq:hybrid_dist}
	\hat{d}(s,a) = m \cdot d_{\mathcal{D}_{\mu}}(s,a) + (1-m) \cdot d_{\mathcal{D}_{\text{pop}}}(s,a).
\end{equation}



\textbf{Correction on Actor}:
Proposition~\ref{prop:mixed_policy_gradient} shows that off-policy population experiences introduce a regularization term in the policy gradient, modifying the objective. When the on-policy ratio $m$ increases, the weights of the regularized term diminish, reducing its effect. When $m=1$, the regularized term from the population experiences will exert no influence on the target actor's updates.

\textbf{Correction on Critic}:
The theorem from \cite{sinhaExperienceReplayLikelihoodfree2022} establishes that the Bellman operator $\mathcal{B}^\mu$ is only a contraction under the distance metric $\| \cdot \|_d$ with the on-policy distribution $d_\mu$, where the distance metric between Q functions under distribution $d$ is defined as $\| Q - Q' \|_d^2 \doteq \mathbb{E}_{(s,a) \sim d} \left[ Q(s,a) - Q'(s,a) \right]^2$, as in the critic update of \eqref{eq:offpolicy_q_update}. Any other training state-action distribution could negatively affect the critic's convergence speed and exacerbate the subsequent actor's learning. As $m$ approaches 1, the mixed distribution $\hat{d}(s,a)$ in \eqref{eq:hybrid_dist} converges to the near-on-policy distribution $d_{\mathcal{D}{\mu}}(s,a)$, which, aided by target network smoothing \cite{lillicrapContinuousControlDeep2019,fujimotoAddressingFunctionApproximation2018}, closely approximates the true on-policy distribution $d_\mu$. In other words, increasing $m$ can help stabilize the learning of the target critic.


Our design choice of using two separate replay buffers offers several advantages. One alternative approach could be to increase the number of rollout workers for the target actor. However, our method reduces the computational cost, and using a fixed ratio in every update prevents fluctuations in the experience distribution due to varying lengths of trajectories from different actors.


\section{Experiments}
\label{sec:experiments}


Based on the results from previous experiments (Section~\ref{sec:empirical_analyses}), we select our tailored ERL framework with the \emph{always} strategy and test it with the remedy method on the same Mujoco locomotion tasks. In order to understand the influence of the target actor data percentage, we experiment with $m$ values of 0.1, 0.25, 0.5, and 0.75, comparing them to the corresponding RL method \emph{no-pop}. Additionally, we compare our ERL method to the parameter-noise method \cite{plappertParameterSpaceNoise2022}, denoted as \emph{param-noise}, which can be regarded as an ERL algorithm with a special Evolution Strategy (ES) algorithm, in which the population mean remains the target actor and does not evolve through the individuals generated from Gaussian noise.

\begin{table*}
	\centering
    \resizebox{\linewidth}{!}{
      \begin{tabular}{*{5}{c}}
         \toprule
          & HalfCheetah-v3 & Hopper-v3 & Ant-v3 & Walker2d-v3 \\ \midrule
         no-pop          & \pmb{$15157.0 \pm 242.9$}                       & $3236.3 \pm 283.0$                     & $5751.2 \pm 216.0$                      & $5324.5 \pm 118.3$                      \\
         param-noise     & $14047.8 \pm 407.9$                       & $3726.1 \pm 22.4$                      & $6014.1 \pm 204.6$                      & $5074.6 \pm 102.7$                      \\
         ERL(m=0.1)  & $13870.7 \pm 465.9$ ($13895.9 \pm 509.8$) & $3373.8 \pm 186.2$ ($3847.9 \pm 73.2$) & $6693.0 \pm 160.3$ ($6677.3 \pm 150.1$) & $4988.1 \pm 163.7$ ($5786.6 \pm 158.9$) \\
         ERL(m=0.25) & $14691.3 \pm 289.2$ ($14649.4 \pm 273.5$) & $3109.9 \pm 277.7$ (\pmb{$3968.1 \pm 22.9$}) & \pmb{$6700.2 \pm 65.1$} (\pmb{$6737.5 \pm 51.6$})   & \pmb{$5534.2 \pm 170.4$} ($6011.5 \pm 201.7$) \\
         ERL(m=0.5)  & $14511.6 \pm 531.1$ ($14506.3 \pm 569.0$) & $3671.4 \pm 55.2$ ($3950.3 \pm 33.0$)  & $5716.3 \pm 264.5$ ($5754.3 \pm 247.8$) & $5528.2 \pm 110.3$ (\pmb{$6149.6 \pm 122.3$}) \\
         ERL(m=0.75) & $13921.5 \pm 327.1$ ($13886.2 \pm 373.7$) & \pmb{$3764.8 \pm 33.3$} ($3959.6 \pm 25.8$)  & $5887.7 \pm 366.4$ ($5876.7 \pm 366.2$) & $5495.4 \pm 109.9$ ($5862.6 \pm 148.7$) \\
         \bottomrule
       \end{tabular}
    }
    \caption{The final performance of no-pop, param-noise and our tailored ERL with the remedy method. The max average return of the last 100 evaluations with the 68\% confidence interval over 10 trials is reported. For ERL algorithms, the performance of the population mean is also reported in brackets. The maximum performance of each task is bolded.
    }
	\label{table:2buffer_results}
 \end{table*}

\textbf{Implementation}:
For our Mujoco experiments, we employ the uniform training pattern described above. We utilize TD3 \cite{fujimotoAddressingFunctionApproximation2018} as our off-policy RL algorithm. Our population size is set to 10, as same as previous works \cite{khadkaEvolutionGuidedPolicyGradient2018,bodnarProximalDistilledEvolutionary2020,pourchotCEMRLCombiningEvolutionary2019}. The sizes of $\mathcal{D}_{\mu}$ and $\mathcal{D}_{\text{pop}}$ are 500,000, while the size of the shared replay buffer for other algorithms is 1,000,000. All tasks are trained for 3 million training steps. Comprehensive training settings can be found in Appendix~\ref{appendix:train}.

\textbf{Metrics}:
Our results are derived from 10 independent trials with different random seeds. For the target actor, we measure its average return across 10 independent episodes and report its 68\% confidence interval over these trials, based on the t-test. The evaluation is conducted every two iterations without exploration noise. Additionally, we record the metrics of the population mean in ERL methods using the same evaluation approach and depict these results with dashed lines. Other metrics are reported at each iteration. The figures are plotted with smoothing.


\subsection{Results}

\begin{figure}[t]
   \centering 
   \begin{subfigure}{0.49\linewidth}
      \centering
      \includegraphics[width=\linewidth]{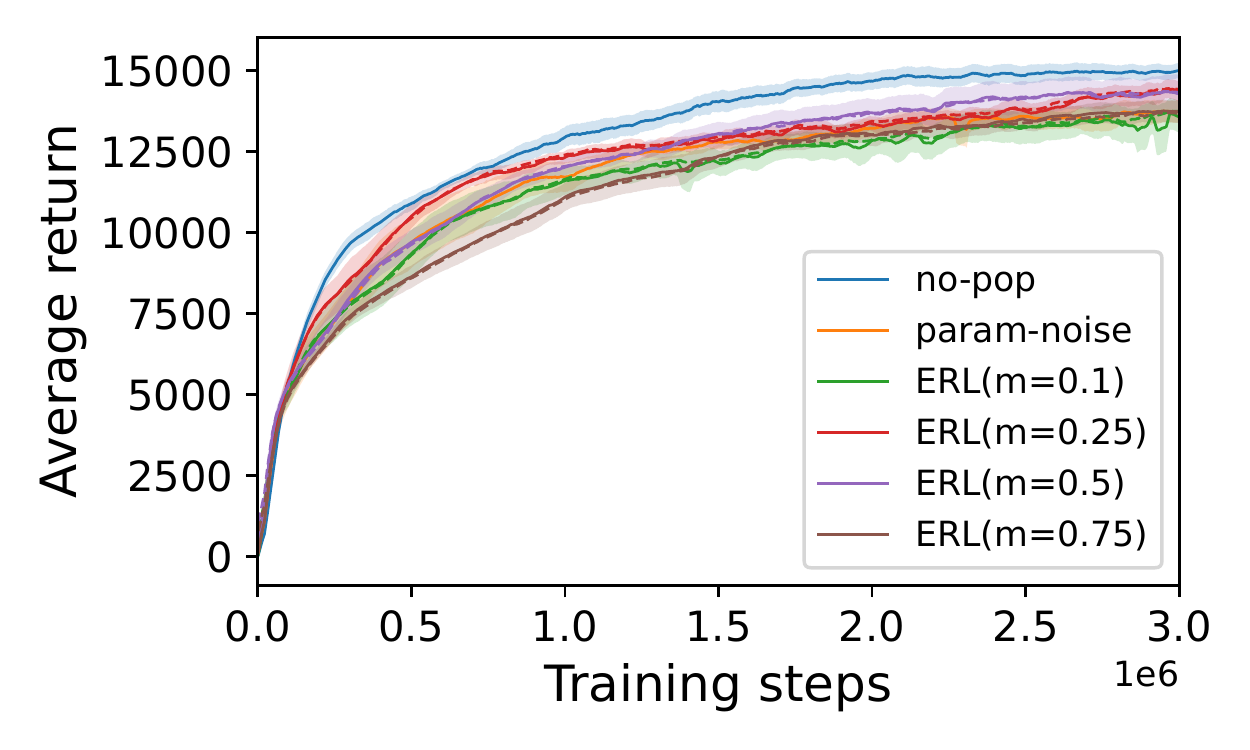} 
      \caption{HalfCheetah-v3}
   \end{subfigure}
   \begin{subfigure}{0.49\linewidth}
      \centering 
      \includegraphics[width=\linewidth]{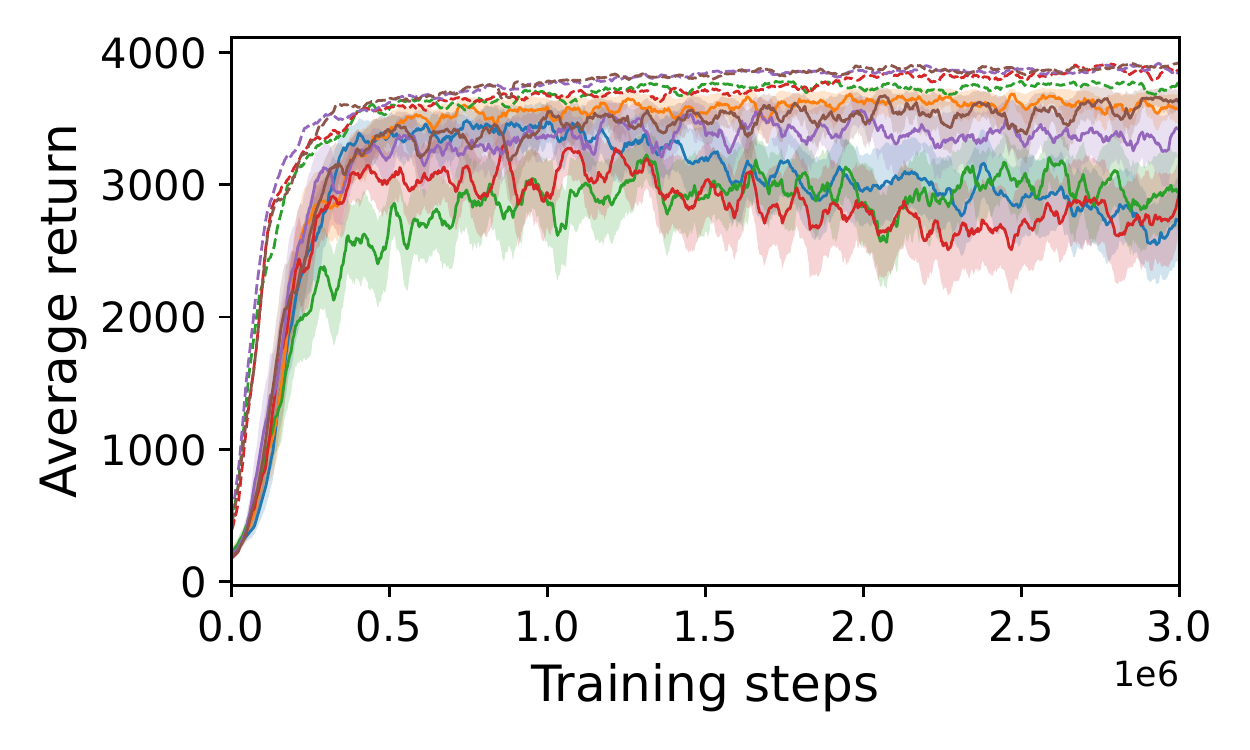} 
      \caption{Hopper-v3}
   \end{subfigure}
   \begin{subfigure}{0.49\linewidth}
       \centering 
       \includegraphics[width=\linewidth]{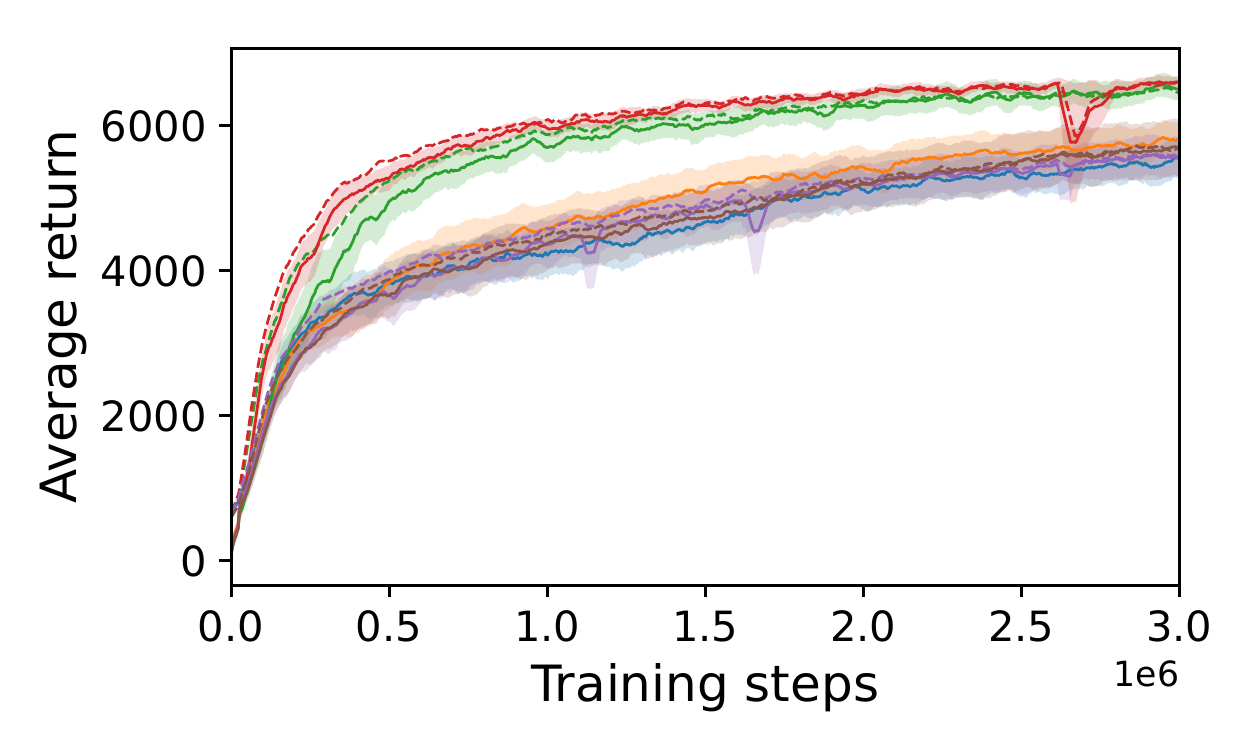} 
       \caption{Ant-v3}
    \end{subfigure}
    \begin{subfigure}{0.49\linewidth}
       \centering 
       \includegraphics[width=\linewidth]{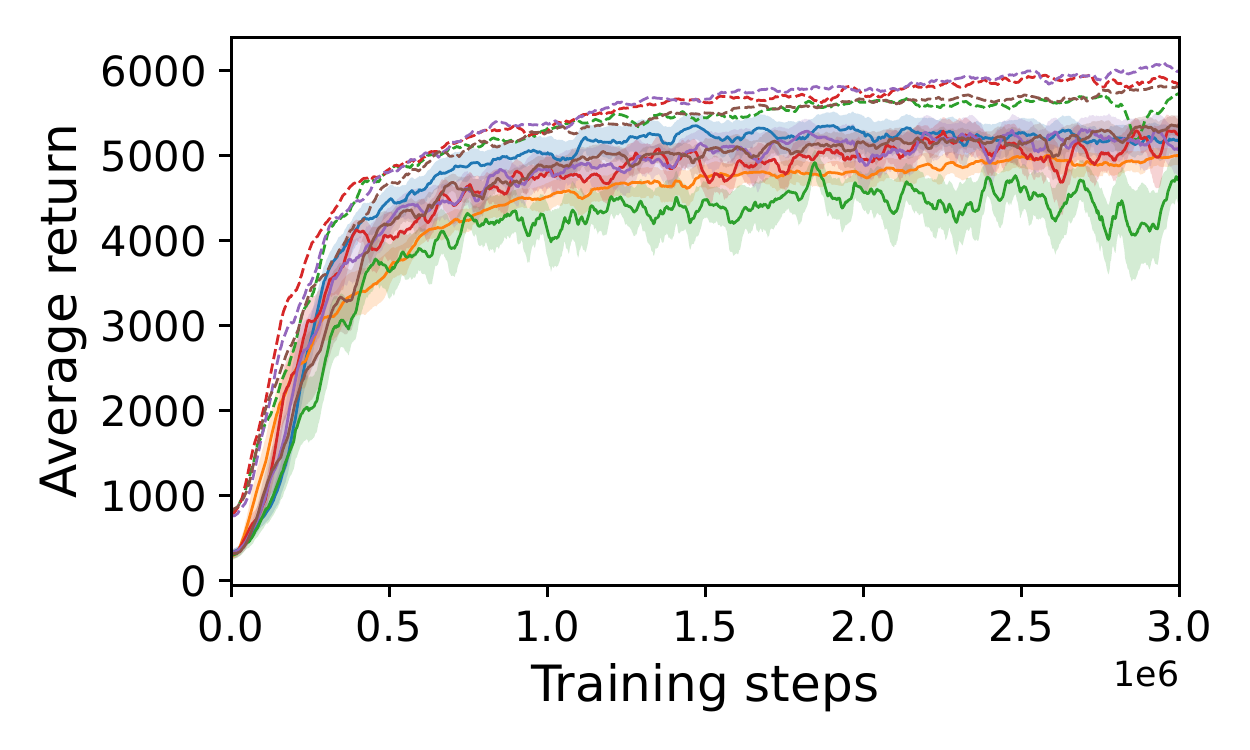} 
       \caption{Walker2d-v3}
    \end{subfigure}
   \caption{Learning curves of no-pop, param-noise and our tailored ERL algorithm with the remedy method on Mujoco tasks. The average return of the population mean in ERL is drawn by dash lines with the same color.}
   \label{fig:erl_2buffer_results}
\end{figure}

The training results are illustrated in Fig.~\ref{fig:erl_2buffer_results}, and the final performance is summarized in Table~\ref{table:2buffer_results}. Section~\ref{sec:empirical_analyses} has revealed that the deviation of population data significantly impacts the performance of the target actor in Hopper-v3 and Walker2d-v3. Our results indicate that increasing the near-on-policy ratio mitigates the negative effects of off-policy population data, resulting in enhanced performance and stability of the target actor's learning. Furthermore, there is a trade-off regarding the population data ratio. For the Ant-v3 task, the results demonstrate that ERL methods surpass \emph{param-noise} and \emph{no-pop} on the target actor when $m$ is small, suggesting that this task benefits from a more expansive exploration strategy with the aid of the population. And our method enables more flexible control over exploration derived from the population. While in HalfCheetah-v3, \emph{no-pop} performs the best among all methods. We hypothesize that broad exploration may not be advantageous in this task, and the Gaussian actors from the population cannot provide more informative data than the target actor, thus hindering the learning of the target actor. And due to the total training steps constraint, ERL algorithms may not converge yet in HalfCheetah-v3.

These results also contribute to a more profound understanding of the population's role in ERL. In Ant-v3, the performance of ERL methods suggests their superiority over traditional parameter-based exploitation methods, as they prioritize high-return areas during exploration. The performance between the population mean and the target actor in Hopper-v3 indicates that, occasionally, the population is more stable than the target actor, providing redundant experiences with high returns for off-policy RL after the target actor temporarily deteriorates. In Hopper-v3 and Walker2d-v3, the population mean exhibits consistently improved performance compared to the target actor, implying that even a simple population-based approach can effectively facilitate policy search with the target actor's assistance. Consequently, we propose using the final population mean as an alternative candidate solution.


\section{Conclusion}

By examining population-assisted off-policy RL methods, we identified a previously overlooked issue: common off-policy RL methods struggle to manage deviated off-policy population experiences generated from evolutionary iterations. Our empirical analyses highlight that even with high-quality population-based methods, the information may not be effectively transferred to the off-policy RL methods via the shared replay buffer. We also stress that our findings apply not only to off-policy deterministic actor-critic algorithms discussed in the paper but also to other approximate off-policy RL methods \cite{tangGuidingEvolutionaryStrategies2021,nguyenCombiningSoftActorCritic2022}, as the mismatched data between the RL policy and the population leads to subtle errors.

To address these errors, we propose a vanilla remedy method involving the integration of more near-on-policy data from the RL policy into the off-policy updates. Our work emphasizes the critical role that on-policy data plays in population-assisted RL methods and highlights the relationship between population data and RL policy data. 
Our solution does not alter the learning pattern of the off-policy RL method, however, it exhibits limitations in learning thoroughly from the population for certain tasks. Consequently, we advise developing patterns that automatically adjust the on-policy data ratio in updates and exploring new off-policy RL methods capable of more effectively handling population experiences. Meanwhile, the population-based part should consider the off-policy degree towards the RL policy and avoid generating substantially deviated off-policy data.

\begin{acks}
  This work was supported by the Program for Guangdong Introducing Innovative and Entrepreneurial Teams (Grant No. 2017ZT07X386).
\end{acks}

\bibliographystyle{ACM-Reference-Format}
\bibliography{main}

\appendix

\onecolumn

\counterwithin{equation}{section}

\section{Proof of the Proposition~\ref{prop:mixed_policy_gradient}}
\label{appendix:props}

\begin{proposition}
	Mixing off-policy data into the policy gradient with the ratio $\alpha$ will changes the deterministic policy gradient from $\mathbb{E}_{s \sim d_\mu} \left[ \nabla_\theta Q^\mu(s, \mu_{\theta}(s))  \right]$ to $\mathbb{E}_{s \sim d_\mu} \left[ \nabla_\theta Q^\alpha(s,\mu_{\theta}(s)) \right]$, where $\rho(s)=\frac{d_b(s)}{d_\mu(s)}$ and
    \begin{equation}
        Q^\alpha(s,a)=Q^\mu(s,a)+\alpha(\rho(s)-1)Q^\mu(s,a).
    \end{equation}
\end{proposition}

\begin{proof}
	We start with the hybrid data distribution $\hat{d}(s)= (1-\alpha) d_\mu(s) + \alpha d_b(s)$ after mixing, where $d_\mu(s)$ and $d_b(s)$ are the marginal state distribution under $\mu$ and $b$ respectively. Then the off-policy policy gradient under $\hat{d}(s)$ can be written as
    \begin{equation}
        \begin{aligned}
            \nabla J_{\hat{d}}(\theta) = & \sum_s \hat{d}(s) \nabla_\theta Q^\mu(s,\mu_{\theta}(s)) \\
            = & \sum_s  [(1-\alpha) d_\mu(s) + \alpha d_b(s)] \nabla_\theta \mu_\theta(s) \nabla_a Q^\mu(s,a)|_{a=\mu_\theta(s)} \\
            = & \sum_s d_\mu(s) \left[  (1-\alpha)+\alpha \frac{d_b(s)}{d_\mu(s)}\right] \nabla_\theta \mu_\theta(s) \nabla_a Q^\mu(s,a)|_{a=\mu_\theta(s)} \\
            = & \sum_s d_\mu(s) \nabla_\theta \mu_\theta(s) \nabla_a \left[(1-\alpha)+\alpha \frac{d_b(s)}{d_\mu(s)}\right] Q^\mu(s,a)|_{a=\mu_\theta(s)} \\
            = & \sum_s d_\mu(s) \nabla_\theta \mu_\theta(s) \nabla_a \left[1+\alpha \frac{d_b(s)-d_\mu(s)}{d_\mu(s)}\right] Q^\mu(s,a)|_{a=\mu_\theta(s)} \\
            = & \sum_s d_\mu(s) \nabla_\theta \mu_\theta(s) \nabla_a \left[Q^\mu(s,a) + \alpha(\frac{d_b(s)}{d_\mu(s)}-1)Q^\mu(s,a)\right] _{a=\mu_\theta(s)} \\
            = & \mathbb{E}_{s \sim d_\mu}  \left[ \nabla_\theta Q^\alpha(s,\mu_{\theta}(s)) \right]
        \end{aligned}
    \end{equation}
    where $Q^\alpha(s,a)=Q^\mu(s,a)+\alpha(\rho(s)-1)Q^\mu(s,a)$.
\end{proof}



\section{Discussion about the Original GA Method}
\label{appendix:ga}
In Section~\ref{sec:tailored_ERL}, we demonstrated that the Genetic Algorithm (GA) employed in the original ERL could perform worse than the corresponding RL algorithm. This appears to contradict the conclusion in the original ERL \cite{khadkaEvolutionGuidedPolicyGradient2018}. In this section, we will discuss the potential reasons for this discrepancy.

In the original work, the ERL framework utilized DDPG as its off-policy RL method and was compared with the pure DDPG method.  In contrast, our experiments replace DDPG with TD3, which addresses the critic overestimation issue and incorporates additional techniques such as Target Policy Smoothing Regularization. Consequently, the benefits derived from the GA population could be trivial compared to the improvement in the off-policy RL algorithm.

Moreover, we evaluate \emph{ERL-GA} using our proposed uniform training design instead of following the original training strategy. We align the RL algorithm and ERL algorithm with similar sampling and update intervals for fair comparisons, which is missing in the original work's comparisons. Different exploration and exploitation balance strategies could also affect the performance of off-policy RL methods. Additionally, a side effect of our training design is that we allow more frequent evolutions and relax the total number of evolution iterations from hundreds to thousands, which could also influence the results of the original GA method.

The motivation for replacing the original GA method with our Estimation-of-Distribution Algorithm(EDA)-style Evolution Strategy(ES) consists of five points, as mentioned in Section~\ref{sec:tailored_ERL}. First, we want all population trajectories to yield high returns, allowing us to eliminate the factor from low-quality population data, which may lead the off-policy RL method to perform worse. Second, the original GA method contains complex procedures and excessive hyperparameters, some of which are task-dependent, which may introduce irrelevant interference. Third, as discussed in \cite{bodnarProximalDistilledEvolutionary2020}, the original GA method uses traditional heuristic operators, which may not be suitable for neural networks and could generate inferior offspring. Fourth, compared to EDA-style ES, GA methods are more sensitive to the noisy evaluation of individuals at each iteration. Lastly, our ES method enables us to derive three update strategy variants that control the distribution mismatch degree between the target actor and population actors, allowing us to qualitatively analyze the effect of the population data in different scenarios.


\section{Training Setup}
\label{appendix:train}

We attempt to make our experiments transparent. In this section, we go into detail about the training setup in our experiments.

For the TD3 algorithm and actor and critic structures, we follow the official code of TD3\footnote{\url{https://github.com/sfujim/TD3}}. Besides, we add fixed layer-normalization layers before each activation function in the actor model. This helps to improve the effectiveness of the noise-based operators in the evolution \cite{plappertParameterSpaceNoise2022}. The model structures and hyperparameters of TD3 are listed in Table~\ref{table:model_struct} and Table~\ref{table:td3_hp} respectively. We use these hyperparameters for all experiments in the paper.

For the tailored ERL algorithm, our population size is 10 and the standard deviation of the Gaussian noise used in the ES is $\sigma=0.01$, with $K=5$. The same Gaussian noise is used in \emph{param-noise} with a fixed standard deviation $\sigma$ of 0.01. The size of $\mathcal{D}_{\mu}$ and $\mathcal{D}_{\text{pop}}$ is 500,000 and the size of the shared replay buffer in other algorithms is 1,000,000. For \emph{ERL-GA} in Section~\ref{sec:tailored_ERL}, we use TD3 with the same hyperparameters mentioned above. And the period of weights injection from RL to EA (Lamarckian transfer) is 10 and 1 for HalfCheetah-v3 and Ant-v3 respectively. Since ERL-GA uses different \emph{evaluation episodes for fitness} and \emph{elite fraction} settings on Hopper and Walker2d tasks, we did not test these two tasks for fair comparisons. The hyperparameters are summarized in Table~\ref{table:ea_hp}. All tasks are trained for 3 million training steps of the target actor and repeated 10 times with different random seeds.

\begin{table}[h]
    \caption{Hyperparameters}
    \begin{subtable}[t]{0.49\linewidth}
        \centering
        \caption{Agent Networks}
        \begin{tabular}[t]{cc}
            \toprule
            Hyperparameter & Value \\ \midrule 
            Actor network & FC(256,256)  \\ 
            Actor activate function & ReLU \\ 
            Actor output function & Tanh \\
            Actor layer normalization & True \\  
            Critic network & FC(256,256)  \\
            Critic activation function & ReLU \\ 
            Critic layer normalization & False \\  \bottomrule
        \end{tabular}
        \label{table:model_struct}
    \end{subtable}
    \begin{subtable}[t]{0.49\linewidth}
        \centering
        \caption{TD3 Training Hyperparameters}
        \begin{tabular}[t]{cc}
            \toprule
            Hyperparameter & Value \\ \midrule
            Actor optimizer & Adam  \\
            Actor learning rate & $3 \cdot 10^{-4}$  \\
            Actor regularization & None \\ 
            Actor delayed update frequency & 2 \\ 
            Critic optimizer & Adam  \\ 
            Critic learning rate & $3 \cdot 10^{-4}$  \\
            Critic regularization & None \\ 
            Batch size & 256 \\ 
            Discount factor & 0.99 \\
            Exploration strategy & $\mathcal{N}(0,0.1)$ \\
            Target network update rate & $5 \cdot 10^{-3}$ \\ 
            Normalize observation & False \\ \bottomrule
        \end{tabular}
        \label{table:td3_hp}
    \end{subtable}
    \begin{subtable}[t]{0.49\linewidth}
        \centering
        \caption{ERL Hyperparameters}
        \begin{tabular}[t]{cc}
            \toprule
            Hyperparameter & Value \\ \midrule
            Population size $N$ & 10  \\
            Evaluation episodes for fitness & 1 \\
            Population standard deviation $\sigma$ (ES)  & 0.01  \\
            Parent size $K$ (ES) & 5 \\ 
            Mutation probability (GA) & 0.9 \\
            Mutation fraction (GA) & 0.1 \\
            Mutation strength (GA) & 0.1 \\
            Super mutation probability (GA) & 0.05 \\
            Reset mutation probability (GA) & 0.1 \\
            Elite fraction (GA) & 0.1 \\ 
            Lamarckian transfer period (GA) & HalfCheetah: 10, Ant: 1 \\ \bottomrule
        \end{tabular}
        \label{table:ea_hp}
    \end{subtable}
	\label{table:hp}
\end{table}

\end{document}